\documentclass[a4paper,10pt]{article}
\usepackage{amsmath,amsthm,amsfonts}
\usepackage{hyperref}

\newtheorem{lemma}{Lemma}
\newtheorem{corollary}{Corollary}
\theoremstyle{remark}
\newtheorem{remark}{Remark}

\usepackage{color}

\newcommand{\rest}{\hat{\pi}}
\newcommand{\der}{\mathrm{d}}

\DeclareMathOperator{\E}{\mathbb{E}}
\DeclareMathOperator{\argmin}{arg\,min}

\title{PAC-Bayes Mini-tutorial:\\A Continuous Union Bound\footnote{This
is an extended version of the blog post at
\url{http://www.timvanerven.nl/blog/2013/12/pac-bayes-mini-tutorial-a-continuous-union-bound/}}}
\author{Tim van Erven}

\begin{document}
\maketitle

\begin{abstract}
  When I first encountered PAC-Bayesian concentration inequalities they
  seemed to me to be rather disconnected from good old-fashioned results
  like Hoeffding's and Bernstein's inequalities. But, at least for one
  flavour of the PAC-Bayesian bounds, there is actually a very close
  relation, and the main innovation is a continuous version of the union
  bound, along with some ingenious applications. Here's the gist of
  what's going on, presented from a machine learning perspective.
\end{abstract}

\section{The Cram\'er-Chernoff Method}

I will start by outlining the Cram\'er-Chernoff method, from which
Hoeffding's and Bernstein's inequalities and many others follow. This
method is incredibly well explained in Appendix~A of the textbook by
Cesa-Bianchi and Lugosi \cite{CesaBianchiLugosi2006}, but I will have to
change the presentation a little to easily connect with the PAC-Bayesian
bounds later on.

Let $D = ((X_1,Y_1),\ldots,(X_n,Y_n))$ be \emph{independent, identically
distributed} (i.i.d.) examples, and let $h$ be a \emph{hypothesis} from
a set of hypotheses $\mathcal{H}$, which gets loss $\ell(X_i,Y_i,h)$ on
the $i$-th example. For example, we might think of the squared loss
$\ell(X_i,Y_i,h) = (Y_i - h(X_i))^2$. We also define the \emph{empirical
error}\footnote{Called the empirical \emph{risk} in statistics; hence
the notation with `R'.} of $h$
\begin{equation*}
  R_n(D,h) = \frac{1}{n} \sum_{i=1}^n \ell(X_i,Y_i,h),
\end{equation*}
and our goal is to prove that the empirical error is close to the
\emph{generalisation error}
\begin{equation*}
  R(h) = \E[\ell(X,Y,h)]
\end{equation*}
with high probability. To do this, we define the function 
\begin{equation*}
  M_\eta(h)
    = -\tfrac{1}{\eta} \ln \E\Big[e^{-\eta \ell(X,Y,h)}\Big]
  \qquad \text{for $\eta > 0$,}
\end{equation*}
which will act as a surrogate for $R(h)$. Now the Cram\'er-Chernoff
method tells us that:
\begin{lemma}\label{lem:chernoff}
  For any $\eta > 0$, $\delta \in (0,1]$,
  \begin{equation}\label{eqn:chernoff}
    M_\eta(h) \leq R_n(h,D) + \frac{1}{\eta n}\ln \frac{1}{\delta}
  \end{equation}
  with probability at least $1-\delta$.
\end{lemma}

\begin{proof}
  By Markov's inequality the probability that
  \begin{equation}\label{eqn:chernoffproof}
    e^{-\eta nR_n(D,h)} \geq \E_{D'}\Big[e^{-\eta nR_n(D',h)}\Big]/\delta
  \end{equation}
  is at most $\delta$. Now, as the examples are i.i.d., we have
  \begin{equation}\label{eqn:iid}
    \E_{D'}\Big[e^{-\eta nR_n(D',h)}\Big] = \E\Big[e^{-\eta
    \ell(X,Y,h})\Big]^n.
  \end{equation}
  Plugging this in and rewriting, we find
  that \eqref{eqn:chernoffproof} is the complement of the event
  \eqref{eqn:chernoff}, from which the result follows.
\end{proof}

It remains to relate $M_\eta(h)$ to $R(h)$, which can be done in
different ways, and then to optimize $\eta$.

\subsection{Specialisations}\label{sec:specialisations}

\paragraph{Hoeffding's Inequality}

To get Hoeffding's inequality, we use Hoeffding's bound
\cite[Lemma~A.1]{CesaBianchiLugosi2006}:
\begin{lemma}[Hoeffding]\label{lem:hoeffding}
  Suppose $\ell(X,Y,h) \in [a,b]$. Then
  \begin{equation*}
    R(h) \leq  M_\eta(h) + \eta \frac{(b-a)^2}{8}.
  \end{equation*}
\end{lemma}
Plugging this into \eqref{eqn:chernoff} gives
\begin{equation*}
  R(h) \leq R_n(D,h) + \eta \frac{(b-a)^2}{8} + \frac{1}{\eta n}\ln
    \frac{1}{\delta},
\end{equation*}
with probability at least $1-\delta$. Then plugging in the choice $\eta =
\sqrt{\frac{8\ln (1/\delta)}{n(b-a)^2}}$, which optimizes the bound,
yields
\begin{equation*}
  R(h) \leq R_n(D,h) + \sqrt{\frac{\ln (1/\delta)(b-a)^2}{2 n}}
\end{equation*}
with probability at least $1-\delta$. This is Hoeffding's inequality
stated `inside out'; to recover the usual formulation, define $\epsilon
= n\sqrt{\frac{\ln (1/\delta)(b-a)^2}{2 n}}$ and solve for $\delta$ in
terms of $\epsilon/n$, leading to
\begin{equation*}
  R(h) \leq R_n(D,h) + \frac{\epsilon}{n}
\end{equation*}
with probability at least $1-\exp\Big\{-2\epsilon^2/\big(n(b-a)^2
\big)\Big\}$.

\paragraph{An Alternative Variance-type Inequality}

There is another inequality that I want to highlight, which is closely
related to Bernstein's inequality. It is derived by plugging in the
following bound, which is essentially Lemma~10 from my NIPS 2012 paper
\cite{VanErvenGrunwaldReidWilliamson2012}:
\begin{lemma}\label{lem:secondmoment}
  Suppose $\ell(X,Y,h) \geq a$ for some $a \leq 0$. Then, for any $\eta
  \in (0,v]$,
  \begin{equation*}
    R(h) \leq M_\eta(h) + \eta \phi(-v a) \E[\ell(X,Y,h)^2].
  \end{equation*}
  where $\phi(x) = (e^x - x - 1)/x^2$ for $x \neq 0$ and $\phi(0) = 1/2$.
\end{lemma}

In particular, if $a = 0$, then $\phi(-va)=\phi(0) = 1/2$ for all $v$,
so we can take $v$ to be infinity.

\begin{proof}
  Let $Z = \ell(X,Y,h)$. Then, by $-\ln x \geq 1-x$, it is sufficient
  to show that
  \begin{equation}\label{eqn:toshow}
    \E[Z] \leq \tfrac{1}{\eta}\Big(1- \E\Big[e^{-\eta Z}\Big]\Big)
      + \eta \phi(-a \eta) \E[Z^2].
  \end{equation}
  Suppose that $\E[Z^2] = 0$. Then $Z = 0$ a.s., and \eqref{eqn:toshow}
  holds with equality. Otherwise, it may be rewritten as
  \begin{equation*}
    \E\Big[\frac{(\eta Z)^2}{\E[(\eta Z)^2]} \cdot \phi(-\eta Z)\Big]
    \leq \phi(-a \eta) 
  \end{equation*}
  Recognising the left-hand side as the expectation of $\phi(-\eta Z)$
  under the distribution with density $\frac{(\eta Z)^2}{\E[(\eta Z)^2]}$
  with respect to the original distribution of $Z$, we see that it can
  be bounded by $\max_z \phi(-\eta z)$. As $\phi$ is increasing, the
  maximum is achieved at the minimum $z = a$ and $\eta = v$,
  from which the desired result follows.
\end{proof}

Combining Lemma~\ref{lem:secondmoment} with Lemma~\ref{lem:chernoff}, we
find that, if $\ell(X,Y,H) \geq a$ for some $a \leq 0$, then for any
$\eta \in (0,v]$
\begin{equation*}
  R(h)
    \leq R_n(D,h) + \eta \phi(-va)\E[\ell(X,Y,h)^2] + \frac{1}{\eta n}\ln
    \frac{1}{\delta},
\end{equation*}
with probability at least $1-\delta$. Optimizing $\eta$ over its allowed
range gives a bound with the flavour of Bernstein's inequality, except
that we don't necessarily require $\ell(X,Y,h)$ to have mean $0$.

\paragraph{Other Standard Inequalities}

As explained in Appendix~A of Cesa-Bianchi and Lugosi
\cite{CesaBianchiLugosi2006}, different bounds to relate $M_\eta(h)$ to
$R(h)$ lead to other inequalities, like for example Bennett's inequality
or the standard version of Bernstein's inequality.

\section{The Union Bound}

Let us get back to the big picture of Lemma~\ref{lem:chernoff} before
its specialisations from the previous section. Now suppose we use an
estimator $\hat{h} \equiv \hat{h}(D) \in \mathcal{H}$ to pick a
hypothesis based on the data, for example using empirical risk
minimization: $\hat{h} = \argmin_{h \in \mathcal{H}} R_n(D,h)$. To get a
bound for $\hat{h}$ instead of a fixed $h$, we want \eqref{eqn:chernoff}
to hold for all $h \in \mathcal{H}$ simultaneously. If $\mathcal{H}$ is
countable, this can be done using the union bound:
\begin{lemma}\label{lem:uniformchernoff}
  Suppose $\mathcal{H}$ is countable. For $h \in \mathcal{H}$, let
  $\pi(h)$ be any numbers such that $\pi(h) \geq 0$ and $\sum_h \pi(h)
  = 1$. Then, for any $\eta > 0$, $\delta \in (0,1]$,
  \begin{equation}
    M_\eta(\hat{h}) \leq R_n(D,\hat{h}) + \frac{1}{\eta n}\ln
    \frac{1}{\pi(\hat{h})\delta}
  \end{equation}
  with probability at least $1-\delta$.
\end{lemma}
In this context, the function $\pi$ is often referred to as a
\emph{prior distribution}, even though it need not have anything to do
with prior beliefs.
\begin{proof}
  By the union bound and Lemma~\ref{lem:chernoff} we have
  \begin{align*}
    \Pr\Big(M_\eta(\hat{h}) &> R_n(D,\hat{h}) + \frac{1}{\eta n}\ln
    \frac{1}{\pi(\hat{h})\delta}\Big)\\
    &\leq \Pr\Big(\exists h: M_\eta(h) > R_n(D,h) + \frac{1}{\eta n}\ln
    \frac{1}{\pi(h)\delta}\Big)\\
    &\leq \sum_h \Pr\Big(M_\eta(h) > R_n(D,h) + \frac{1}{\eta n}\ln
    \frac{1}{\pi(h)\delta}\Big)
    \leq \sum_h \pi(h)\delta = \delta.\qedhere
  \end{align*}
\end{proof}
Just like for Lemma~\ref{lem:chernoff}, we can then again relate
$M_\eta(h)$ to $R(h)$ to obtain a bound on the generalisation error, but
there is now a slight complication: when we want to optimize $\eta$, we
find that we are not allowed to, because the optimal choice of $\eta$
depends on $\hat{h}$, which depends on the data, whereas
Lemma~\ref{lem:chernoff} only allows a fixed choice of $\eta$. In some
applications using a fixed $\eta$ may be good enough, but this does
limit the applicability of the result. Luckily, it turns out that we can
optimize $\eta$ ``for free'':
\begin{lemma}\label{lem:uniformchernoffeta}
  Suppose $\mathcal{H}$ is countable. For $h \in \mathcal{H}$, let
  $\pi(h)$ be any numbers such that $\pi(h) \geq 0$ and $\sum_h \pi(h)
  = 1$. Then, for any $\delta \in (0,1]$,
  \begin{equation}
    M_\eta(\hat{h}) \leq R_n(D,\hat{h}) + \frac{1}{\eta n}\ln
    \frac{1}{\pi(\hat{h})\delta}
    \qquad \text{for all $\eta > 0$}
  \end{equation}
  with probability at least $1-\delta$.
\end{lemma}
\begin{proof}
  Let $\eta(h) = \argmin_{\eta > 0} \frac{1}{\eta n}\ln
  \frac{1}{\pi(h)\delta} - M_\eta(\hat{h})$ be the optimal value for
  $\eta$ if $\hat{h} = h$. Now apply Lemma~\ref{lem:uniformchernoff}
  with $\eta = 1$ and the scaled loss $\ell'(X,Y,h) =
  \eta(h)\ell(X,Y,h)$ to obtain
  \begin{equation}
    - \ln \E\Big[e^{-\eta(\hat{h}) \ell(X,Y,\hat{h})}\Big]
    \leq \eta(\hat{h})R_n(D,\hat{h}) + \frac{1}{n}\ln
    \frac{1}{\pi(\hat{h})\delta}
  \end{equation}
  with probability at least $1-\delta$. Dividing both sides by
  $\eta(\hat{h})$ gives the result for the choice of $\eta$ that
  optimizes the bound. It follows that the bound holds simultaneously
  for all other $\eta$ as well.
\end{proof}

This shows, in a nutshell, how one can combine the Cram\'er-Chernoff
method with the union bound to obtain concentration inequalities for
estimators $\hat{h}$. The use of the union bound, however, is quite
crude when there are multiple hypotheses in $\mathcal{H}$ with very
similar losses, and the current proof breaks down completely if we want
to extend it to continuous classes $\mathcal{H}$. This is where
PAC-Bayesian bounds come to the rescue: in the next section I will
explain the PAC-Bayesian generalisation of
Lemmas~\ref{lem:uniformchernoff} and \ref{lem:uniformchernoffeta} to
continuous hypothesis classes $\mathcal{H}$, which will require
replacing $\hat{h}$ by a randomized estimator.

\section{PAC-Bayesian Concentration}

Let $\rest \equiv \rest(D)$ be a distribution on $\mathcal{H}$ that
depends on the data $D$, which we will interpret as a randomized
estimator: instead of choosing $\hat{h}$ deterministically, we will
sample $h \sim \rest$ randomly. The distribution $\rest$ is often called
the PAC-Bayesian \emph{posterior distribution}. Now the result that the
PAC-Bayesians have, may be expressed as follows:
\begin{lemma}\label{lem:pacbayes}
  Let $\pi$ be a (prior) distribution on $\mathcal{H}$ that does not
  depend on $D$, and let $\rest$ be a randomized estimator that is
  allowed to depend on $D$. Then, for any $\eta > 0$, $\delta \in
  (0,1]$,
  \begin{equation}\label{eqn:pacbayes}
    \E_{h \sim \rest}[M_\eta(h)] \leq \E_{h \sim \rest}[R_n(D,h)] +
    \frac{1}{\eta n}\Big(D(\rest\|\pi) + \ln \frac{1}{\delta}\Big)
  \end{equation}
  with probability at least $1-\delta$. Moreover,
  \begin{equation}\label{eqn:pacbayesexp}
    \E_D \E_{h \sim \rest}[M_\eta(h)] \leq \E_D\Big[ \E_{h \sim \rest}[R_n(D,h)] +
    \frac{1}{\eta n}D(\rest\|\pi)\Big].
  \end{equation}
\end{lemma}
Here $D(\rest\|\pi) = \int \rest(h) \ln \frac{\rest(h)}{\pi(h)}\der h$
denotes the Kullback-Leibler divergence of $\rest$ from $\pi$.

\begin{proof}[Proof of Lemma~\ref{lem:pacbayes}]
  By \eqref{eqn:iid}, we have
  \begin{equation*}
    e^{-\eta n M_\eta(h)} = \E_D\Big[e^{-\eta nR_n(D,h)}\Big].
  \end{equation*}
  Hence
  \begin{align*}
    1 &= \E_{h \sim \pi}\E_D \Big[\exp\Big\{-\eta nR_n(D,h) + \eta
    n M_\eta(h)\Big\}\Big]\\
      &= \E_D \E_{h \sim \pi} \Big[\exp\Big\{-\eta nR_n(D,h) + \eta
    n M_\eta(h)\Big\}\Big]\\
      &= \E_D \E_{h \sim \rest} \Big[\exp\Big\{-\eta nR_n(D,h) + \eta
    n M_\eta(h) - \ln \frac{\rest(h)}{\pi(h)}\Big\}\Big]\\
      &\geq \E_D\Big[\exp\Big\{\underbrace{-\eta n\E_{h \sim
      \rest}[R_n(D,h)] + \eta n \E_{h \sim \rest}[M_\eta(h)] -
      D(\rest\|\pi)}_A\Big\}\Big],
  \end{align*}
  where the inequality is Jensen's. Now notice that
  \eqref{eqn:pacbayes} is equivalent to $A \leq \ln (1/\delta)$,
  whereas \eqref{eqn:pacbayesexp} is equivalent to $\E[A] \leq 0$, and
  that we have derived that $\E[e^A] \leq 1$.
  \eqref{eqn:pacbayes} therefore follows by Markov's inequality:
  \begin{equation*}
    \Pr\Big(A > \ln\frac{1}{\delta}\Big) = \Pr(e^A > 1/\delta) \leq
    \E[e^A]\delta \leq \delta,
  \end{equation*}
  and \eqref{eqn:pacbayesexp} follows by another application of Jensen's
  inequality:
  \begin{equation*}
    e^{\E[A]} \leq \E\big[e^A\big] \leq 1
    \qquad \Longrightarrow \qquad \E[A] \leq 0. \qedhere
  \end{equation*}
\end{proof}

To see that Lemma~\ref{lem:pacbayes} generalises
Lemma~\ref{lem:uniformchernoff}, suppose that $\rest$ is a point-mass on
$\hat{h}$. Then $D(\rest\|\pi) = \ln (1/\pi(\hat{h}))$, and we recover
Lemma~\ref{lem:uniformchernoff} as a special case of
\eqref{eqn:pacbayes}. An important difference with
Lemma~\ref{lem:uniformchernoff}, however, is that
Lemma~\ref{lem:pacbayes} does not require $\mathcal{H}$ to be countable,
and in fact in many PAC-Bayesian applications it is not.

\subsection{Optimizing $\eta$}

Lemma~\ref{lem:pacbayes} has the same issue as
Lemma~\ref{lem:uniformchernoff}; namely that it does not allow us to
optimize $\eta$ based on $\rest$. For the result in expectation
\eqref{eqn:pacbayesexp} I do not really know how to introduce
optimization over $\eta$ in a satisfying way, and we are stuck with a
fixed $\eta$. For the result in probability \eqref{eqn:pacbayes} we
cannot use the same trick that allowed us to optimize $\eta$ ``for
free'' in Lemma~\ref{lem:uniformchernoffeta}, but we can still optimize
$\eta$ at very small cost using the union bound as long as we can find a
good lower bound on its range:
\begin{lemma}\label{lem:pacbayeseta}
  For any constants $\alpha > 1$ and $0 < u < v$, and any $\delta \in
  (0,1]$,
  \begin{multline}\label{eqn:pacbayeseta}
    \E_{h \sim \rest}[M_\eta(h)] \leq \E_{h \sim \rest}[R_n(D,h)] +
    \frac{\alpha}{\eta n}\Big(D(\rest\|\pi) + \ln \frac{1}{\delta} +
    \ln \Big\lceil\log_\alpha\frac{v}{u}\Big\rceil\Big)\\
    \text{for all $\eta \in [u,v]$}
  \end{multline}
  with probability at least $1-\delta$.
\end{lemma}
\begin{proof}
  For $i = 0,\ldots, \Big\lceil\log_\alpha\frac{v}{u}\Big\rceil-1$, let
  $\eta_i = u\alpha^i$. Then for every $\eta \in [u,v]$, there exists an
  $\eta_i$ such that $\eta_i \leq \eta \leq \alpha \eta_i$. Using the
  union bound to
  extend \eqref{eqn:pacbayes} to hold uniformly over all $\eta_i$, we
  find that
  \begin{multline*}
    \E_{h \sim \rest}[M_{\eta_i}(h)] \leq \E_{h \sim \rest}[R_n(D,h)] +
    \frac{1}{\eta_i n}\Big(D(\rest\|\pi) + \ln \frac{1}{\delta}
     + \ln \Big\lceil\log_\alpha\frac{v}{u}\Big\rceil\Big)\\
      \text{for all $\eta_i$}
  \end{multline*}
  with probability at least $1-\delta$. Now we use that $M_\eta(h)$
  is nonincreasing in $\eta$, so that, for any $\eta \in [u,v]$ and
  $\eta_i$ such that $\eta_i \leq \eta \leq \alpha \eta_i$, we have
  $M_\eta(h) \leq M_{\eta_i}(h)$ and $\frac{1}{\eta_i} \leq
  \frac{\alpha}{\eta}$, from which the lemma follows.
\end{proof}

Having an upper bound on the range of $\eta$ is not an issue, because
\begin{equation*}
  \min_{0 < \eta \leq v} \Big(\eta A + \frac{B}{\eta}\Big)
    \leq \min_{\eta > 0} \Big(\eta A + \frac{B}{\eta}\Big) + \frac{2
    B}{v} \qquad \text{for $A,B > 0$},
\end{equation*}
which only adds the term $\frac{2B}{v}$, which is always negligible in
our case. So it remains to find a good lower bound $u$ for $\eta$ to
plug into Lemma~\ref{lem:pacbayeseta}. I don't know of a general
procedure to do that, but after applying the specialisations from
Section~\ref{sec:specialisations} it actually becomes easy:
\begin{lemma}[PAC-Hoeffding]\label{lem:pachoeffding}
  Suppose $\ell(X,Y,h) \in [a,b]$. Then, for any constants $\alpha > 1$
  and $v>0$, and any $\delta \in (0,1]$,
  \begin{multline}\label{eqn:pachoeffding}
    \E_{h \sim \rest}[R(h)] \leq\\ \E_{h \sim \rest}[R_n(D,h)]
    + \eta \frac{(b-a)^2}{8}
    + \frac{\alpha}{\eta n}\Big(D(\rest\|\pi) + \ln \frac{1}{\delta} +
    \ln (\tfrac{1}{2}\log_\alpha n + C)\Big)\\
    \text{for all $\eta \in (0,v]$}
  \end{multline}
  with probability at least $1-\delta$, where $C = \max\{\log_\alpha
  \big(\frac{v(b-a)}{\sqrt{8\alpha}}\big),0\}+e$.
\end{lemma}

\begin{proof}
  Combining Lemma~\ref{lem:pacbayeseta} with Lemma~\ref{lem:hoeffding},
  we find for any $u \in (0,v)$
  \begin{multline*}
    \E_{h \sim \rest}[R(h)] \leq\\ \E_{h \sim \rest}[R_n(D,h)]
    + \eta \frac{(b-a)^2}{8}
    + \frac{\alpha}{\eta n}\Big(D(\rest\|\pi) + \ln \frac{1}{\delta} +
    \ln \Big\lceil\log_\alpha\frac{v}{u}\Big\rceil\Big)\\
    \text{for all $\eta \in [u,v]$}
  \end{multline*}
  with probability at least $1-\delta$. Using that $C \geq e$, the
  unconstrained value for $\eta$ that optimizes \eqref{eqn:pachoeffding}
  can be bounded from below by
  \begin{equation*}
    \eta = \sqrt{\frac{8\alpha\Big(D(\rest\|\pi) + \ln \frac{1}{\delta}
    + \ln (\tfrac{1}{2}\log_\alpha n + C) \Big)}{n(b-a)^2}}
      \geq \sqrt{\frac{8\alpha}{n(b-a)^2}},
  \end{equation*}
  which does not depend on $h$. So now we choose $u =
  \frac{1}{\sqrt{n}}\min\{\sqrt{\frac{8\alpha}{(b-a)^2}},v\}$, from
  which the desired result follows.
\end{proof}

\begin{lemma}[PAC-Variance]\label{lem:pacvariance}
  Suppose $\ell(X,Y,h) \in [a,b]$ with $a \leq 0$. Then, for any
  constants $\alpha > 1$ and $v>0$, and any $\delta \in (0,1]$,
  \begin{multline*}
    \E_{h \sim \rest}[R(h)] \leq\\ \E_{h \sim \rest}[R_n(D,h)]
    + \eta \phi(-av) \E[\ell(X,Y,h)^2]
    + \frac{\alpha}{\eta n}\Big(D(\rest\|\pi) + \ln \frac{1}{\delta} +
    \ln (\tfrac{1}{2}\log_\alpha n + C)\Big)\\
    \text{for all $\eta \in (0,v]$}
  \end{multline*}
  with probability at least $1-\delta$, where $C = \max\{\tfrac{1}{2}\log_\alpha
  \big(\frac{v\max\{a^2,b^2\}\phi(-av)}{\alpha}\big),0\}+e$.
\end{lemma}

\begin{proof}
  Analogously to the proof of Lemma~\ref{lem:pachoeffding}, combine
  Lemma~\ref{lem:secondmoment} with Lemma~\ref{lem:pacbayeseta}, and now
  observe that the minimizing $\eta$ is at least $\sqrt{\frac{\alpha }{n
  \phi(-av) \max\{a^2,b^2\}}}$. Then pick $u =
  \frac{1}{\sqrt{n}}\min\{\sqrt{\frac{\alpha}{\phi(-av)\max\{a^2,b^2\}}},v\}$
  to obtain the result.
\end{proof}

\section{Corollaries}

Because Lemma~\ref{lem:pacbayes} works for any choice of loss, we may in
particular plug in the relative loss $\ell'(X,Y,h) = \ell(X,Y,h) -
\ell(X,Y,h^*)$, where $h^* = \argmin_{h \in \mathcal{H}} R(h)$ is the
hypothesis with smallest generalisation error in $\mathcal{H}$.
Combining this, for example, with the PAC-Bayesian version of
Hoeffding's lemma (Lemma~\ref{lem:pachoeffding}), we obtain:
\begin{corollary}
  Suppose $\ell(X,Y,h) \in [0,b]$, so that $\ell'(X,Y,h) \in [-b,b]$.
  Then, for any constants $\alpha > 1$ and $v>0$, and any $\delta \in
  (0,1]$,
  \begin{multline}
    \E_{h \sim \rest}[R(h)]-R(h^*)\\
      \leq \E_{h \sim \rest}\big[R_n(D,h)\big] - R_n(D,h^*)
      + \eta \frac{b^2}{2}
      + \frac{\alpha}{\eta n}\Big(D(\rest\|\pi) + \ln \frac{1}{\delta} +
      \ln (\tfrac{1}{2}\log_\alpha n + C)\Big)\\
      \text{for all $\eta \in (0,v]$}
    \end{multline}
    with probability at least $1-\delta$, where $C = \max\{\log_\alpha
    \big(\frac{2vb}{\sqrt{8\alpha}}\big),0\}+e$.
\end{corollary}
And combining with Lemma~\ref{lem:pacvariance}, we get:
\begin{lemma}\label{lem:excessvariance}
  Suppose $\ell(X,Y,h) \in [0,b]$, so that $\ell'(X,Y,h) \in [-b,b]$.
  Then, for any constants $\alpha > 1$ and $v>0$, and any $\delta \in
  (0,1]$,
  \begin{multline*}
    \E_{h \sim \rest}[R(h)]-R(h^*) \leq\\ \E_{h \sim \rest}[R_n(D,h)] -
    R_n(D,h^*)
    + \eta \phi(bv) \E[\ell'(X,Y,h)^2]\\
    + \frac{\alpha}{\eta n}\Big(D(\rest\|\pi) + \ln \frac{1}{\delta} +
    \ln (\tfrac{1}{2}\log_\alpha n + C)\Big)
    \qquad\text{for all $\eta \in (0,v]$}
  \end{multline*}
  with probability at least $1-\delta$, where $C = \max\{\tfrac{1}{2}\log_\alpha
  \big(vb^2\phi(bv)/\alpha\big),0\}+e$.
\end{lemma}

\section{Choosing the Prior and the Posterior}

Even though the names prior and posterior for $\pi$ and $\rest$ suggest
some kind of fixed relationship between the two, all the previous
results actually hold for any way of choosing these two distributions.
This is exploited in applications, in which there appear to be two main
approaches:

\paragraph{Optimal Posterior}

In the first approach, the prior $\pi$ is fixed, and the posterior
$\rest$ is chosen as the distribution that optimizes the bound. In
Lemmas~\ref{lem:pacbayeseta}--\ref{lem:excessvariance} this is always
the \emph{Gibbs distribution}
\begin{equation}\label{eqn:posterior}
  \rest(h) = \frac{e^{-\frac{\eta}{\alpha} nR_n(D,h)} \pi(h)}{\int
  e^{-\frac{\eta}{\alpha} n R_n(D,h')}
  \pi(h') \der h'}.
\end{equation}

\paragraph{Localised Priors}

By contrast, in the second approach the posterior $\rest$ is fixed, and
then the prior $\pi$ is chosen to (almost) optimize the bound. This way
of selecting $\pi$ was developed by Catoni \cite{Catoni2007}, who refers
to such $\pi$ as \emph{localised priors}. For given $\rest$, the prior
that exactly optimizes the bound\footnote{At a NIPS 2013 workshop David
McAllester referred to this as ``Langford's prior'', because apparently
John Langford already observed that it optimized the bound 13 years ago,
but I don't have a reference.} is
\begin{equation}\label{eqn:optimalprior}
  \pi(h) = \E_D[\rest(h)],
\end{equation}
but when the posterior takes the form \eqref{eqn:posterior} another
common choice, for which the bound becomes easier to manipulate, is the
prior $\pi'$ defined by
\begin{equation*}
  \pi'(h) = \frac{e^{-\frac{\eta}{\alpha} nR(h)}\pi(h)}{\int
  e^{-\frac{\eta}{\alpha}
  nR(h')}\pi(h')\der h'}
\end{equation*}
for $\pi$ from the definition of $\rest$.

\begin{remark}
  For given prior, the posterior \eqref{eqn:posterior} minimizes the
  bound, and, for given posterior, the prior \eqref{eqn:optimalprior}
  minimizes the bound. Since both steps reduce the bound, we can
  conceivably iterate them until convergence. I wonder whether there
  exists any stability point $\pi$ such that
  \begin{equation*}
    \pi(h) = \E_D\Big[\frac{e^{-\frac{\eta}{\alpha} nR_n(D,h)} \pi(h)}{\int
    e^{-\frac{\eta}{\alpha}
    nR_n(D,h')}
  \pi(h') \der h'}\Big],
  \end{equation*}
  and, if so, whether it is unique.
\end{remark}

\section{Summary}

We have seen how PAC-Bayesian inequalities naturally extend standard
concentration inequalities based on the Cram\'er-Chernoff method by
generalising the union bound to a continuous version. There are some
technicalities involved if we want to optimize over $\eta$, but these
can be managed if we can find a good lower bound on the value of the
optimizing $\eta$. I have not discussed any applications, for which I
will have to refer to the references discussed next.

\section{Further Reading}

I learned about PAC-Bayesian concentration inequalities by discussions
with Peter Gr\"unwald about papers by Zhang \cite{Zhang2006}, and by
reading the (quite technical) monograph of Catoni \cite{Catoni2007}. For
a much more accessible presentation of Catoni's idea of localised priors
and their applications, see the recent paper by Lever, Laviolette and
Shawe-Taylor \cite{LeverLST2013}. McAllester also has a recent tutorial
\cite{McAllester2013}, which includes an application to analysing
drop-out. Except for the connection to standard concentration
inequalities, which is probably well known, but which I have not seen
emphasised before, all the results I have presented here can be found
(more or less) in these references. For more advanced concentration
inequalities based on the Cram\'er-Chernoff method, I also highly
recommend the recent textbook by Boucheron, Lugosi and Massart
\cite{BoucheronLM2013}, which I am sure will be a classic.

\bibliographystyle{abbrv}
\bibliography{pacbayes}

\end{document}